\documentclass{article}

\usepackage[preprint]{neurips_2026}

\usepackage[utf8]{inputenc}
\usepackage[T1]{fontenc}
\usepackage{hyperref}
\usepackage{url}
\usepackage{booktabs}
\usepackage{amsfonts}
\usepackage{amsmath,amssymb,amsthm}
\usepackage{mathtools}
\usepackage{nicefrac}
\usepackage{microtype}
\usepackage{xcolor}
\usepackage{graphicx}
\usepackage{enumitem}
\usepackage{wrapfig}
\usepackage{subcaption}
\usepackage{float}
\usepackage{fontawesome5}

\graphicspath{{./}}

\newtheorem{theorem}{Theorem}
\newtheorem{proposition}[theorem]{Proposition}

\newtheorem{corollary}[theorem]{Corollary}
\theoremstyle{definition}
\newtheorem{definition}[theorem]{Definition}
\theoremstyle{remark}

\newcommand{\R}{\mathbb{R}}
\newcommand{\E}{\mathbb{E}}
\newcommand{\cN}{\mathcal{N}}
\newcommand{\cL}{\mathcal{L}}
\newcommand{\smin}{\sigma_{\mathrm{min}}}
\newcommand{\neff}{n_{\mathrm{eff}}}
\newcommand{\dmodel}{d_{\mathrm{model}}}
\newcommand{\dk}{d_k}

\title{Support-Conditioned Flow Matching Is Kernel Smoothing}

\author{
  Daniel Matsui Smola \\
  Department of Computer Science \\
  University of Washington \\
  Seattle, WA \\
  \texttt{daniel@matsuismola.com} \\
}

\begin{document}

\maketitle

\vspace{-0.5em}
{\centering\footnotesize
  \href{https://github.com/BaroqueObama/kernel-flow-matching-code}{{\color{black}\faGithub~Code}}
\par\vspace{0.5em}}

\begin{abstract}
Generative models are often conditioned on a small set of examples via cross-attention. Under the Gaussian optimal-transport path, we show that the exact velocity field induced by a finite support set is a Nadaraya--Watson kernel smoother whose bandwidth decreases with flow time, from broad averaging at early steps to nearest-neighbor at late steps. A single Gaussian-kernel attention head exactly computes this field, connecting cross-attention conditioning to classical kernel theory. The theory predicts three failure regimes: nearest-neighbor collapse of the kernel at high dimension, mismatch between the isotropic kernel and the data geometry, and insufficient support for nonparametric estimation. Experiments on Gaussian mixtures, spherical shells, and DINOv2 ImageNet features confirm that learned conditioning improves in precisely these regimes, and that IP-Adapter's cross-attention implements approximate NW smoothing in practice.
\end{abstract}

\section{Introduction}
\label{sec:intro}

Text alone often cannot specify the visual attributes one wants to generate in an image: a particular face, a beloved pet, art style, or object appearance. The standard solution is to condition on reference examples via cross-attention: IP-Adapter \citep{ye2023ipadapter}, Context Diffusion \citep{ivanova2024context}, SuTI \citep{chen2023suti}, and few-shot diffusion models \citep{giannone2022fewshot} all condition the generative process by attending to a set of reference image features.

However, this mechanism is understood only empirically. One may ask: \emph{What is the mathematical structure of the velocity field that a finite reference set induces?} \emph{What governs the success or failure of cross-attention conditioning?}

We answer these questions in the flow matching setting. Given a finite support set of samples from an unknown target distribution, the task is to learn a velocity field that generates fresh samples via ODE integration without parameter updates at test time. We show that the exact velocity field induced by such a support set under the Gaussian optimal-transport path is a Nadaraya--Watson (NW) kernel smoother \citep{nadaraya1964estimating, watson1964smooth}. This amounts to a weighted average of the support points, with Gaussian kernel weights that concentrate more heavily on nearby points as the flow progresses. A single Gaussian-kernel cross-attention head computes this field exactly (Theorem~\ref{thm:realization}), for any support set and any flow time.

Classical kernel smoothing has well-characterized failure modes, and each appears in the conditioning setting:
\begin{enumerate}[nosep]
  \item \textbf{Nearest-neighbor collapse.} At high dimension, concentration of measure collapses the Gaussian kernel weights to a single nearest neighbor. The exact field reduces to nearest-neighbor interpolation, and using it as an inductive bias becomes harmful.
  \item \textbf{Geometry mismatch.} The isotropic kernel smoothes uniformly in all directions. When variation in the data is concentrated along a low-dimensional subspace, the kernel under-smoothes the informative directions while wasting bandwidth on uninformative ones, yielding high-variance velocity estimates independently of support size.
  \item \textbf{Support scarcity.} At small support sizes, nonparametric estimation is data-starved. Learned models improve by amortizing over a meta-distribution, exploiting cross-task structure that single-task kernel estimation cannot access.
\end{enumerate}
We verify each failure mode experimentally on Gaussian mixtures, spherical shells, and DINOv2 ImageNet features, and show that learned conditioning improves precisely where the theory predicts the exact field is deficient. IP-Adapter's attention weights approximate NW kernel weights in practice (\S\ref{sec:adapters}), suggesting the framework applies beyond our controlled setting.

\paragraph{Relation to prior work.}
\citet{tsai2019transformer} and \citet{goel2024can} show that attention \emph{can} implement kernel regression. We show that the OT-FM velocity field \emph{is} a kernel smoother---the attention construction is one implementation of a structure already present in the velocity. \citet{fukumizu2025flow} and \citet{kunkel2025minimax} derive minimax rates from the endpoint KDE connection; \citet{zhou2026discretization} analyzes the empirical velocity and discretization bias. We derive the full-path velocity-field decomposition, of which the endpoint KDE is a special case, and experimentally verify its predicted failure modes.

\paragraph{Contributions.}
\begin{itemize}[nosep]
  \item We derive the NW velocity identity (Proposition~\ref{prop:velocity}).
  \item We prove one cross-attention head computes this field exactly (Theorem~\ref{thm:realization}).
  \item We identify three failure modes from kernel theory and confirm each experimentally (\S\ref{sec:experiments}).
  \item We show that multi-head attention decomposes into $H$ NW smoothers in $\dk$ dimensions (Proposition~\ref{prop:multihead}).
  \item We show IP-Adapter's attention weights approximate NW kernel weights (\S\ref{sec:adapters}).
\end{itemize}

\section{Background and setup}
\label{sec:setup}

Flow matching learns a velocity field $v(x, t)$ that transports samples from a noise distribution ($t = 0$) to a target data distribution ($t = 1$) via ODE integration $\mathrm{d}x/\mathrm{d}t = v(x, t)$. To generate, one draws noise $X_0 \sim \cN(0, I_d)$ and integrates forward to obtain a sample at $t = 1$.

\paragraph{Optimal-transport (OT) flow matching.}
Given a data distribution $q$ on $\R^d$ and base $p_0 = \cN(0, I_d)$, the OT conditional path \citep{lipman2023flow} is
\begin{equation}\label{eq:ot-path}
  X_t = t\,X_1 + \sigma_t\,X_0, \qquad \sigma_t = 1 - (1-\smin)\,t,
\end{equation}
where $X_1 \sim q$, $X_0 \sim \cN(0,I_d)$, and $0 < \smin \le 1$. The conditional flow matching objective regresses the conditional velocity target $Y_t = X_1 - (1-\smin)X_0$.

We define the \emph{de-scaled bandwidth} $h(t) = \sigma_t/t$, which is strictly decreasing on $(0,1]$, sweeping from $\infty$ (at $t \to 0$) to $\smin$ (at $t = 1$). As we will show, this quantity connects flow matching time to kernel density estimation bandwidth.

\paragraph{Nadaraya--Watson kernel regression.}
For bandwidth $h>0$ and support set $S = \{s_i\}_{i=1}^m$, the Gaussian kernel weights and local mean are
\begin{equation}\label{eq:nw-weights}
  w_i(x) = \frac{\phi_h(x - s_i)}{\sum_{j=1}^m \phi_h(x - s_j)}, \qquad m_h(x; S) = \sum_{i=1}^m w_i(x)\,s_i,
\end{equation}
where $\phi_h(z) = (2\pi h^2)^{-d/2}\exp(-\|z\|^2/2h^2)$. The effective sample size $\neff = 1/\sum_i w_i^2$ measures how many support points meaningfully contribute.

\paragraph{Problem: support-conditioned generation.}
Given $S = \{s_i\}_{i=1}^m \overset{\text{iid}}{\sim} q$ from an unknown $q$, the goal is to learn a velocity field $v_\theta(x, t, S)$ that generates fresh samples from $q$ via ODE integration, without parameter updates at test time. We call this \emph{In-Context Flow Matching} (ICFM) and train with the support-conditioned loss:
\begin{equation}\label{eq:icfm-loss}
  \cL(\theta) = \E_{q \sim \Pi}\;\E_{S \sim q^m}\;\E_{X_1 \sim q,\,X_0 \sim \cN(0,I_d),\,t \sim \tau}\bigl[\|v_\theta(X_t, t, S) - Y_t\|^2\bigr],
\end{equation}
where $\Pi$ is a meta-distribution over data distributions and $\tau = \mathrm{Uniform}(0,1]$.

We distinguish three velocity fields. The \emph{population field} $u_t^q$ is the marginal velocity induced by the true $q$. The \emph{empirical plug-in field} $u_t^S$ replaces $q$ by $\hat{q}_S = m^{-1}\sum_i \delta_{s_i}$ and is exactly computable for any support set via a single Gaussian-kernel attention head (\S\ref{sec:realization}). The \emph{learned field} $v_\theta$ is trained over tasks from $\Pi$. The empirical question is whether $v_\theta$ improves over $u_t^S$, and if so, whether the improvement can be explained by identifiable failure modes of the plug-in field.

\section{The NW velocity identity}
\label{sec:theory}

Full proofs are in Appendix~\ref{app:proofs}; we state the key results with proof sketches here.

\subsection{The OT path as a bandwidth sweep}
\label{sec:kde}

We show that the OT probability path through a finite support set is a kernel density estimate (KDE) whose bandwidth decreases with time. The key step is \emph{de-scaling}: dividing $X_t = tX_1 + \sigma_t X_0$ by $t$ gives $\tilde{X}_t \coloneqq X_t/t = X_1 + (\sigma_t/t)\,X_0$. Since $X_0 \sim \cN(0,I_d)$, this is $X_1$ plus Gaussian noise at standard deviation $h(t) = \sigma_t/t$. When the data distribution is the empirical measure $\hat{q}_S = \frac{1}{m}\sum_i \delta_{s_i}$ (uniform point masses at the support points), the de-scaled density is a Gaussian KDE:

\begin{proposition}[De-scaled empirical OT path is a KDE]\label{prop:kde}
Under the Gaussian OT conditional path~\eqref{eq:ot-path}, for every $t \in (0,1]$, the de-scaled empirical density $\tilde{p}_t^S(\tilde{x}) \coloneqq t^d\,p_t^S(t\tilde{x})$ equals
\begin{equation}\label{eq:kde}
  \tilde{p}_t^S(\tilde{x}) = \frac{1}{m}\sum_{i=1}^m \phi_{h(t)}(\tilde{x} - s_i), \qquad h(t) = \frac{\sigma_t}{t}.
\end{equation}
As $t$ traverses $(0,1]$, $h(t)$ sweeps from $\infty$ to $\smin$: the OT path through a support set is a continuous multi-resolution KDE, from maximum smoothing to near-delta peaks at the support points.
\end{proposition}
\begin{proof}
Using the Gaussian scaling identity $t^d\,\phi_{\sigma_t}(t\,z) = \phi_{\sigma_t/t}(z)$:
$t^d\,p_t^S(t\tilde{x}) = \frac{1}{m}\sum_i t^d\,\phi_{\sigma_t}(t\tilde{x} - ts_i) = \frac{1}{m}\sum_i \phi_{h(t)}(\tilde{x} - s_i)$.
\end{proof}

\textbf{Implication.} At each flow time $t$, the generative process maintains a KDE of the support set at bandwidth $h(t)$. The ODE generates samples from a continuously sharpening density estimate: broad smoothing at early times, near-delta peaks at the support points as $t \to 1$. The noise schedule $\sigma_t$ therefore determines the bandwidth schedule $h(t) = \sigma_t/t$: for practitioners, designing the noise schedule for support-conditioned generation is equivalent to selecting a time-varying kernel bandwidth, a problem with a rich classical literature. The velocity field that drives this sharpening is characterized next.

\subsection{The exact velocity field is a NW smoother}
\label{sec:velocity}

The marginal velocity field is $u_t(x) = \E[X_1 - (1-\smin)X_0 \mid X_t = x]$. For the Gaussian OT path, the Tweedie identity \citep{efron2011tweedie} connects this to the score, and de-scaling converts the score into NW form.

\begin{proposition}[Exact empirical OT-FM velocity field]\label{prop:velocity}
Under the Gaussian OT conditional path~\eqref{eq:ot-path} with $X_0 \sim \cN(0,I_d)$, the exact empirical velocity field induced by support set $S$ is:
\begin{equation}\label{eq:velocity}
  u_t^S(x) = \tilde{x} + \frac{m_{h(t)}(\tilde{x}; S) - \tilde{x}}{\sigma_t}, \qquad \tilde{x} = x/t,\quad h = h(t),
\end{equation}
where $m_{h(t)}(\tilde{x}; S) = \sum_i w_i(\tilde{x})\,s_i$ is the NW local mean with Gaussian kernel weights at bandwidth $h(t) = \sigma_t/t$.
\end{proposition}
\begin{proof}[Proof sketch]
Applying Tweedie's identity to the signal-plus-noise decomposition $X_t = tX_1 + \sigma_t X_0$ gives $\E[X_1 \mid X_t = x] = x/t + (\sigma_t^2/t)\,\nabla\log p_t(x)$. Substituting into the marginal velocity and simplifying yields the velocity--score relation:
\begin{equation}\label{eq:vel-score}
  u_t(x) = \frac{x}{t} + \frac{\sigma_t}{t}\,\nabla\log p_t(x).
\end{equation}
Differentiating the KDE from Proposition~\ref{prop:kde} gives the score in NW form: $\nabla\log\tilde{p}_t^S(\tilde{x}) = (m_h(\tilde{x};S) - \tilde{x})/h^2$. Evaluating~\eqref{eq:vel-score} at $x = t\tilde{x}$ and substituting yields:
\[
  u_t^S(t\tilde{x}) = \tilde{x} + \frac{\sigma_t}{t^2}\cdot\frac{m_h(\tilde{x};S) - \tilde{x}}{h(t)^2}.
\]
Since $h(t) = \sigma_t/t$, we have $\sigma_t/(t^2 h^2) = 1/\sigma_t$, yielding~\eqref{eq:velocity}.
\end{proof}

\textbf{Implication.} The velocity in~\eqref{eq:velocity} consists of an inertia term $\tilde{x}$ and a correction $(m_h - \tilde{x})/\sigma_t$ that steers toward the NW mean of nearby support points. As $t \to 1$ and $\sigma_t \to \smin$, this correction dominates: at late flow times, the field pulls samples strongly toward the kernel-weighted average of their neighbors. The exact velocity is therefore entirely determined by the NW estimator $m_h$. If the kernel weights degenerate (e.g., collapsing to a single neighbor) or the kernel shape mismatches the data geometry, the velocity field inherits these defects---it steers samples incorrectly, degrading generation quality. We characterize these failures in \S\ref{sec:experiments}.

\subsection{Attention realization}
\label{sec:realization}

Since the NW local mean is a weighted average with softmax-normalized weights, it can be computed by a cross-attention head with appropriately chosen logits:

\begin{theorem}[Attention realization of the empirical OT-FM field]\label{thm:realization}
Under the Gaussian OT conditional path~\eqref{eq:ot-path}, consider a single cross-attention head with query $\tilde{x}$, keys and values $\{s_i\}_{i=1}^m$, and Gaussian-kernel logits $\ell_i = -\|\tilde{x} - s_i\|^2/(2h(t)^2)$. Define the affine post-map $A_t(\tilde{x}, z) = \tilde{x} + (z - \tilde{x})/\sigma_t$. Then
\begin{equation}\label{eq:realization}
  A_t\!\bigl(\tilde{x},\;\mathrm{Attn}(\tilde{x}, t, S)\bigr) = u_t^S(t\,\tilde{x}).
\end{equation}
That is, one Gaussian-kernel cross-attention head followed by an affine map exactly implements the marginal velocity field induced by $S$.
\end{theorem}
\begin{proof}[Proof sketch]
The softmax weights with logits $\ell_i = -\|\tilde{x} - s_i\|^2/(2h^2)$ satisfy $\alpha_i = \exp(\ell_i)/\sum_j\exp(\ell_j)$. Since $\phi_h(z) \propto \exp(-\|z\|^2/2h^2)$, the normalization constant $(2\pi h^2)^{-d/2}$ cancels between numerator and denominator: $\alpha_i = \phi_h(\tilde{x} - s_i)/\sum_j \phi_h(\tilde{x} - s_j) = w_i(\tilde{x})$, the NW weights. So $\mathrm{Attn}(\tilde{x},t,S) = \sum_i w_i\,s_i = m_h(\tilde{x};S)$, and $A_t(\tilde{x}, m_h) = \tilde{x} + (m_h - \tilde{x})/\sigma_t = u_t^S(t\tilde{x})$ by Proposition~\ref{prop:velocity}.
\end{proof}

\textbf{Implication.} The exact NW field can be embedded as a frozen, zero-parameter attention head inside any learned model, providing a built-in nonparametric baseline. This gives a concrete experimental tool: if the exact head helps, the NW signal is useful at that operating point; if it makes no difference, the signal is uninformative and the model learns to ignore it. We use this probe in \S\ref{sec:collapse}.

\subsection{Multi-head NW decomposition}
\label{sec:multihead-theory}

In practice, the learned model uses multiple cross-attention heads with trained projection matrices alongside the frozen exact head. The NW framework extends naturally to this setting:

\begin{proposition}[Multi-head attention as NW ensemble]\label{prop:multihead}
Multi-head cross-attention with $H$ heads computes $\mathrm{MHA}(q, Z) = \sum_{h=1}^H W_O^h\,m_{K_h}\!\bigl(q;\,\{(z_i,\,W_V^h z_i)\}\bigr)$, where each $m_{K_h}$ is a generalized NW estimator with kernel $K_h(x,s) = \exp(x^\top W_Q^{h\top} W_K^h s / \sqrt{\dk})$. The logit bilinear form has rank $\le \dk = \dmodel/H$, so each head's weights depend on at most $\dk$ independent linear functionals of the query and keys.
\end{proposition}

The proof follows directly from the definition of multi-head attention (Appendix~\ref{app:multihead}).

\textbf{Implication.} Each learned head computes a NW-type velocity correction in a $\dk$-dimensional subspace. Since the isotropic kernel degrades with ambient dimension (as we will show in \S\ref{sec:collapse}), projecting into lower-dimensional subspaces allows each head to maintain effective smoothing where the full-dimensional kernel would collapse. The learned component of the model is therefore an ensemble of $H$ adaptive smoothers, each avoiding the full curse of dimensionality by operating in a learned $\dk$-dimensional projection of the support set. Formal convergence rates and an anisotropic extension are given in Appendices~\ref{app:multihead} and~\ref{app:aniso}.

\section{Failure modes and what learning corrects}
\label{sec:experiments}

The NW identity from \S\ref{sec:theory} establishes that cross-attention conditioning implements a principled nonparametric estimator. However, the specific isotropic Gaussian kernel underlying the plug-in field has well-known limitations. We now test three predictions from classical NW theory and show what the learned model corrects in each case.

\subsection{Setup}
\label{sec:exp-setup}

All experiments train the ICFM model from~\eqref{eq:icfm-loss}. The velocity network $v_\theta(x, t, S)$ is a transformer ($\dmodel=128$, 3 layers) that conditions on the support set via $H=4$ learned cross-attention heads with trainable $W_Q, W_K, W_V$ projections (the generalized NW smoothers from Proposition~\ref{prop:multihead}). Separately, an optional \textbf{frozen exact head} computes the plug-in velocity from Theorem~\ref{thm:realization} using Gaussian-kernel attention and feeds it as an additional signal into the output MLP. The key experimental manipulation is enabling or disabling this head:
\begin{itemize}[nosep]
  \item \textbf{Plug-in:} the exact NW velocity field integrated alone (no learned parameters).
  \item \textbf{Learned (OFF):} learned heads only. This is closest to what practical systems (IP-Adapter, Context Diffusion) do.
  \item \textbf{Learned (ON):} learned heads + frozen exact head. Our experimental probe, testing whether explicitly providing the NW signal as an inductive bias helps.
\end{itemize}
If ON outperforms OFF, the NW velocity is a useful inductive bias. If ON matches OFF, the signal is uninformative. Sample quality is measured by MMD$^2$ (unbiased Gaussian-kernel U-statistic) and C2ST on 100 held-out tasks ($\ge\!4$ seeds per condition). Full training details are in Appendix~\ref{app:training}.

\subsection{Collapse: the isotropic kernel becomes nearest-neighbor}
\label{sec:collapse}

The isotropic Gaussian kernel in Proposition~\ref{prop:velocity} is subject to the curse of dimensionality. Concentration of measure makes pairwise distance contrasts vanish as $O(1/\sqrt{d})$ \citep{beyer1999when}, and the exponential kernel amplifies even small relative differences into overwhelming weight ratios \citep{bengio2005curse}. The NW weights should therefore collapse to a single nearest neighbor at high $d$.

To see why this degrades the velocity, recall that the plug-in error is $u_t^S - u_t^q = (m_h^S - m_h^q)/\sigma_t$ (from Proposition~\ref{prop:velocity}). The NW variance scales as $O(1/(m\,h^d\,f(\tilde{x})))$ \citep{tsybakov2009introduction}, and dividing by $\sigma_t^2$ gives a velocity MSE of order $O(t^d/(m\,\sigma_t^{d+2}\,f(\tilde{x})))$, which grows exponentially with $d$ (proof in Appendix~\ref{app:vel-mse}). At high $d$, the kernel concentrates on a single neighbor ($\neff \to 1$, where $\neff$ measures how many support points contribute) and the velocity becomes a volatile nearest-neighbor lookup. This is the memorization phenomenon documented in trained diffusion models \citep{lyu2025memorization}. For practitioners, the $\neff$ quantity provides a closed-form, per-timestep diagnostic of this effect without requiring model training. With the kernel collapsed, the signal from the exact head becomes uninformative, and we expect the ON condition to converge toward OFF as $d$ increases.

\textbf{Experiment.} We train on Gaussian mixtures at $d = 2, 4, 8, 16$ with the frozen exact head ON vs.\ OFF ($\ge\!4$ seeds per condition).

\textbf{Results.} Figure~\ref{fig:transition} (left) shows the exact head benefit declining monotonically from $+28\%$ at $d=2$ to $0\%$ at $d=16$, confirming the prediction: the NW signal becomes uninformative as dimension increases. Figure~\ref{fig:transition} (center) confirms the mechanism: $\neff$ drops from ${\sim}9$ at $d=2$ to ${\sim}1$ at $d=16$, consistent with the kernel collapsing to nearest-neighbor.

\textbf{What learning corrects.}\label{sec:multihead} The learned heads avoid collapse by projecting into $\dk$-dimensional subspaces (Proposition~\ref{prop:multihead}), where kernel weights remain well-distributed. Figure~\ref{fig:transition}(c) verifies this: varying head count $H$ at $d=8$ on Gaussian mixtures (exact head OFF, $\dk = \dmodel/H$) and fitting MMD$^2 \sim m^{-\alpha}$, the scaling exponent increases monotonically with $H$ ($\rho = 1.0$, $p = 0.008$). More heads---each operating in a lower-dimensional subspace---yield faster improvement with support size, consistent with each head avoiding the full-dimensional curse.

\begin{figure}[ht!]
  \centering
  \includegraphics[width=\linewidth]{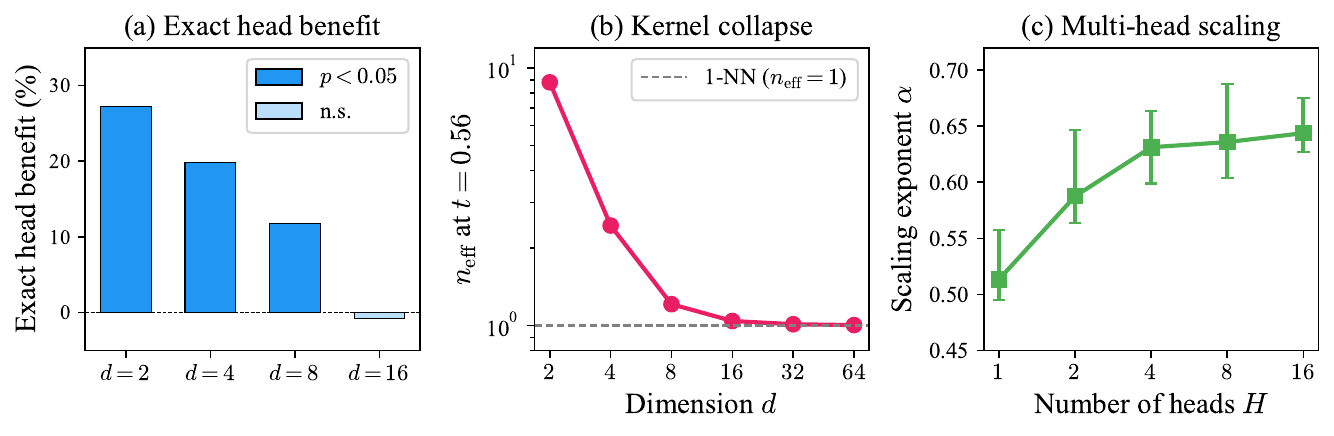}
  \vspace{-18pt}
  \caption{\textbf{(a)} Exact head benefit (\%) vs.\ dimension ($\ge\!4$ seeds): declines from $+28\%$ at $d=2$ to $0\%$ at $d=16$. \textbf{(b)} $\neff$ at mid-flow ($t=0.56$) vs.\ dimension: the kernel collapses to 1-NN over the same range. \textbf{(c)} Multi-head scaling on $\R^8$ Gaussian mixtures: the exponent $\alpha$ (from MMD$^2 \sim m^{-\alpha}$) increases with head count $H$; more heads (each in lower $\dk = \dmodel/H$ dimensions) yield faster rates.}
  \label{fig:transition}
\end{figure}

\subsection{Geometry mismatch: the isotropic kernel wastes bandwidth}
\label{sec:mismatch}

The isotropic kernel can also fail when its shape does not match the data geometry, independently of whether $\neff$ is healthy. An isotropic bandwidth smoothes equally in all directions; when variation in the data is concentrated along a low-dimensional subspace, this wastes capacity on uninformative directions while under-smoothing informative ones \citep{lepski1997optimal}. On distributions with anisotropic structure, the learned model (which can adapt its kernel via $W_Q, W_K$ projections; Proposition~\ref{prop:multihead}) should dramatically outperform the plug-in, even when the plug-in's kernel weights are well-distributed.

\textbf{Experiment.} We test on spherical shells (distributions where all variation is radial, embedded in $\R^d$) at $d = 16$, comparing the learned model (OFF) to the plug-in. To test whether geometry mismatch is independent of $\neff$ collapse, we also evaluate on DINOv2+PCA ImageNet features \citep{oquab2024dinov2} (pre-computed DINOv2 ViT-S/14 class features projected via PCA) and a whitened variant of the same features where whitening makes the covariance isotropic, artificially collapsing $\neff$. If $\neff$ collapse alone were sufficient to make the plug-in fail, whitening should increase the plug-in's MMD$^2$ relative to the learned model's. If geometry mismatch is the independent factor, the ratio should remain near $1$.

\textbf{Results.} Table~\ref{tab:failure-modes} summarizes. On shells, $\neff$ is healthy ($7.1$) yet the plug-in is $11\times$ worse than the learned model: the isotropic kernel smoothes radially and angularly alike, but shells vary only radially. The paired ImageNet control confirms that $\neff$ collapse is not sufficient to cause plug-in failure: whitening collapses $\neff$ from $13.4$ to $1.1$, yet the ratio does not increase ($0.99 \to 0.74$). The plug-in remains effective because nearest neighbors on these features are still same-class members regardless of whitening.

\textbf{What learning corrects.} Figure~\ref{fig:corrects}(a) shows how the learned model overcomes geometry mismatch. On spherical shells at $d=8$, the isotropic NW kernel collapses from $\neff \approx 44$ at $t=0.1$ to $\neff \approx 1$ at $t=0.9$. The 4 learned heads resist this collapse: at late flow times, individual heads range from $\neff \approx 6$ to $\neff \approx 48$, indicating different heads learn different effective bandwidths. Moreover, the cosine similarity between learned and NW attention weights drops from $0.91$ at $t=0.1$ to $0.13$ at $t=0.9$: at late timesteps where isotropic NW fails, the learned model adopts a fundamentally different weighting scheme rather than perpetuating the failure. By maintaining multi-neighbor smoothing and specializing across heads, the learned model avoids the volatile nearest-neighbor velocities that degrade the plug-in.

\begin{table}[h!]
\caption{Three distributions at $d=16$. The ratio is plug-in MMD$^2$ / learned MMD$^2$ (smaller MMD$^2$ is better, so ratio $\gg 1$ means the plug-in fails). Shells have healthy $\neff$ yet the plug-in is $11\times$ worse, demonstrating geometry mismatch as an independent failure mode.}
\label{tab:failure-modes}
\centering
\small
\begin{tabular}{lccrl}
\toprule
Distribution & $d$ & $\neff$ & Plug-in/learned & Interpretation \\
\midrule
Shells & 16 & 7.1 & $\mathbf{10.9\times}$ & Healthy $\neff$ but wrong kernel shape \\
ImageNet & 16 & 13.4 & $0.99\times$ & Plug-in competitive \\
Whitened ImageNet & 16 & 1.1 & $0.74\times$ & $\neff$ collapsed $12\times$; plug-in still competitive \\
\bottomrule
\end{tabular}
\end{table}

\subsection{Support scarcity: when the plug-in is data-poor}
\label{sec:learning}

The plug-in field has an intrinsic quality ceiling. By Proposition~\ref{prop:kde} at $t = 1$, its ODE flow produces samples distributed as $\hat{q}_S * \phi_{\smin}$. This is a Gaussian KDE at bandwidth $\smin$; at small $m$, it is data-poor regardless of kernel shape or dimensionality.

\textbf{Experiment.} We use the DINOv2+PCA ImageNet features introduced in \S\ref{sec:mismatch} at $d=16$, because neither collapse nor geometry mismatch is active at this dimension. This isolates support size as the remaining variable. We train separate models at each support size $m \in \{1, 5, 10, 25, 50\}$, comparing the learned model (OFF) to the plug-in.

\begin{figure}[t]
  \centering
  \includegraphics[width=\linewidth]{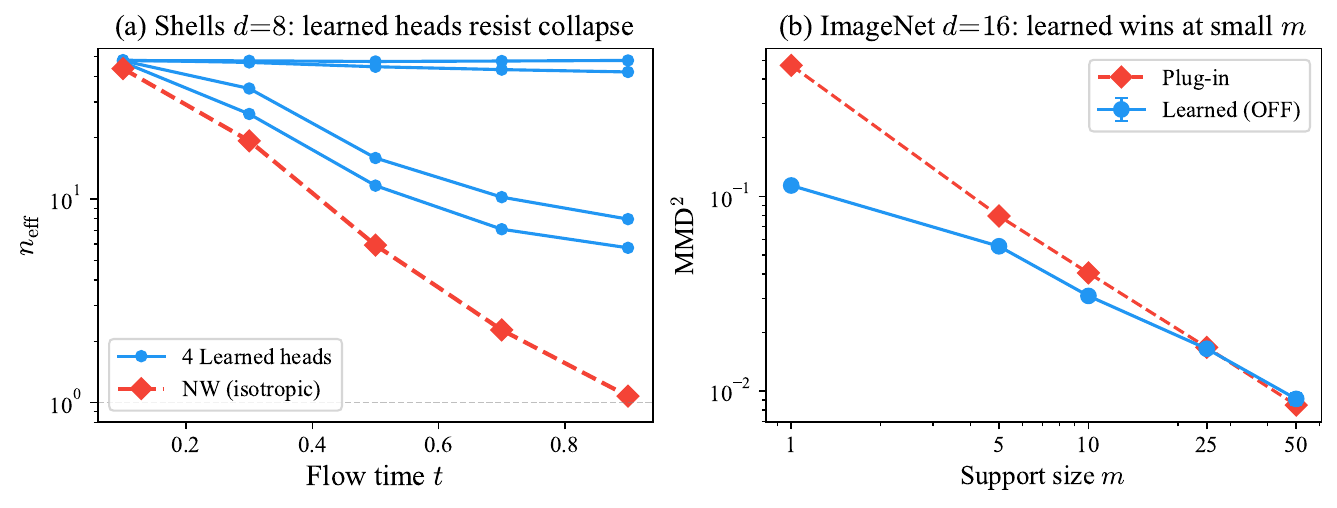}
  \vspace{-14pt}
  \caption{\textbf{(a)} On spherical shells at $d=8$, the isotropic NW kernel collapses to 1-NN at late flow times, while learned heads maintain multi-neighbor smoothing with differentiated bandwidths. \textbf{(b)} On DINOv2+PCA ImageNet features ($d=16$), the learned model outperforms the plug-in at small $m$; the plug-in catches up as $m$ grows. Crossover near $m \approx 25$.}
  \label{fig:corrects}
\end{figure}

\textbf{Results.} Figure~\ref{fig:corrects}(b) shows the learned and plug-in curves crossing near $m \approx 25$. Before this crossover, the learned model outperforms the plug-in: at $m = 10$, it achieves $24\%$ lower MMD$^2$ ($4$ seeds). After the crossover, the plug-in's KDE has enough support points and the learned model offers no additional benefit. The small-support regime ($m \le 10$) is where reference-based adapters operate, typically conditioning on $1$--$4$ reference images.

\textbf{What learning corrects.} The plug-in velocity is the Bayes-optimal predictor for a single task under the empirical measure $\hat{q}_S$ (Appendix~\ref{app:bayes}), but it cannot use information from other tasks. The ICFM loss~\eqref{eq:icfm-loss} trains over many tasks drawn from $\Pi$; the unconstrained minimizer is the posterior predictive velocity $v^*(x,t,S) = \E[Y_t \mid X_t = x, t, S]$, which integrates over all distributions consistent with $S$ under the meta-distribution. When $m$ is small and many distributions are consistent with $S$, this posterior predictive is substantially better than the empirical plug-in because it exploits shared structure across tasks \citep{meunier2023nonlinear}. This meta-learning advantage is what the learned model provides and what single-task kernel estimation cannot access.

\subsection{Controls}
\label{sec:controls}

Replacing the Euler-100 ODE solver with adaptive dopri5 for plug-in integration improves MMD$^2$ by only $2.4\%$ on $\R^2$ Gaussian mixtures, ruling out solver error as the source of plug-in failure. Doubling model capacity ($\dmodel = 256$ vs.\ $128$) changes MMD$^2$ by less than $5\%$ on $\R^4$ Gaussian mixtures, so the dimensionality transition in \S\ref{sec:collapse} is not a capacity bottleneck. Sampling training times $t$ from a log-uniform distribution (which overweights small bandwidths) rather than the uniform distribution used throughout gives no measurable gain ($p = 0.67$, four seeds).

\section{Extension to practical conditioning}
\label{sec:adapters}

The theory above is derived for Gaussian-kernel attention under the Gaussian OT path. Do practical conditioning systems, which use scaled dot-product attention and different noise schedules, exhibit similar NW structure?

\textbf{Experiment.} We test on IP-Adapter \citep{ye2023ipadapter}, a widely used reference-conditioning module for Stable Diffusion 1.5 that conditions on $m = 4$ reference image tokens via cross-attention. IP-Adapter uses a VP diffusion schedule rather than the Gaussian OT path; we index its noise level by the diffusion timestep $t$ (large $t$ = noisy, small $t$ = clean), which plays an analogous role to flow time in our framework. For each attention head at each denoising timestep, we compare the learned SDPA attention weights (dot-product based) to the Gaussian NW kernel weights computed from the same query and key vectors (distance based). We measure Spearman $\rho$ and Pearson $r$ between the two weight vectors across 50 images, 16 layers $\times$ 8 heads, and 5 denoising timesteps.

\textbf{Results.} Figure~\ref{fig:ipadapter}(a) shows that both Spearman $\rho$ and Pearson $r$ increase monotonically with denoising, reaching $\rho = 0.89$ and $r = 0.91$ at the final timestep. This trend is consistent with the bandwidth schedule from Proposition~\ref{prop:kde}: as $h(t) \to \smin$, the NW weights sharpen and the kernel-smoothing component dominates the velocity. Figure~\ref{fig:ipadapter}(b) shows that $56\%$ of heads exceed $\rho = 0.9$ at the final step, up from $22\%$ at the noisiest step.

\textbf{Interpretation.} IP-Adapter uses scaled dot-product attention (not Gaussian-kernel) and a VP noise schedule (not Gaussian OT), so Theorem~\ref{thm:realization} does not formally apply. Nevertheless, the strong monotonic correlation suggests that the NW structure persists approximately under these architectural departures, and that the failure-mode framework of \S\ref{sec:experiments} may extend to practical reference-conditioned generation.

\begin{figure}[ht!]
  \centering
  \includegraphics[width=\linewidth]{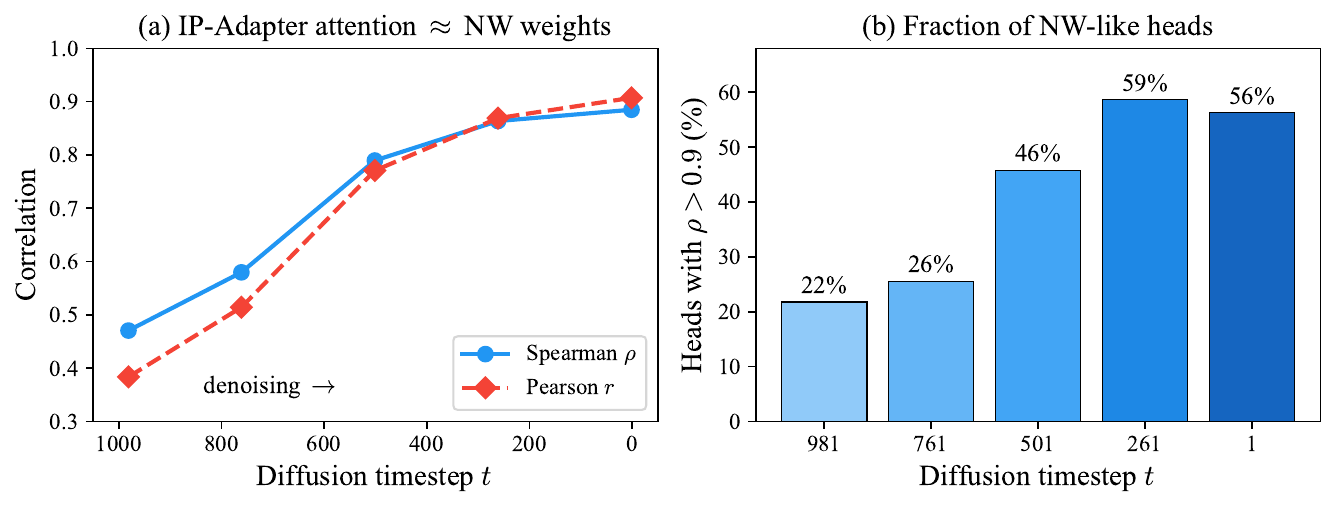}
  \vspace{-14pt}
  \caption{\textbf{(a)} Correlation between IP-Adapter attention weights and Gaussian NW kernel weights, measured across denoising. Both Spearman $\rho$ and Pearson $r$ increase monotonically, reaching $\rho = 0.89$ at the final step. \textbf{(b)} Fraction of attention heads with $\rho > 0.9$; over half of heads are NW-like by the end of denoising.}
  \label{fig:ipadapter}
\end{figure}

\section{Related work}
\label{sec:related}

\paragraph{Attention as kernel smoothing.}
\citet{tsai2019transformer} interpret attention through kernel smoothing; \citet{katharopoulos2020transformers} and \citet{choromanski2021rethinking} develop kernelized attention variants; \citet{goel2024can} show a transformer can represent a Kalman filter, constructing an attention-based kernel regression in the process. Our result is different in kind: the OT-FM velocity field \emph{is} a NW smoother; the attention construction is one implementation of a structure already present in the empirical velocity.

\paragraph{Flow matching theory.}
\citet{lipman2023flow} and \citet{albergo2023stochastic} establish the flow matching framework. \citet{fukumizu2025flow} and \citet{kunkel2025minimax} derive minimax rates via the endpoint KDE connection; \citet{zhou2026discretization} analyzes the empirical velocity and discretization bias; \citet{liu2025oracle} identify oracle-velocity regimes at late timesteps. We derive the full-path bandwidth family $h(t) = \sigma_t/t$ and the velocity-field NW decomposition, of which the endpoint KDE is a special case (\S\ref{sec:kde}).

\paragraph{Nonparametric estimation theory.}
Each failure mode in \S\ref{sec:experiments} has classical roots. The 1-NN collapse follows from concentration of measure \citep{beyer1999when} and exponential kernel amplification \citep{bengio2005curse}. The geometry-mismatch failure connects to anisotropic minimax theory \citep{lepski1997optimal}. The support-scarcity crossover reflects the gap between single-task NW rates \citep{stone1982optimal, tsybakov2009introduction} and meta-learning rates that exploit cross-task structure \citep{meunier2023nonlinear}. The 1-NN collapse also connects to recent memorization analyses in diffusion models: \citet{lyu2025memorization} show that empirical score models behave as nearest-neighbor maps in the low-noise regime; $\neff$ gives a closed-form version of this phenomenon for support-conditioned flows.

\paragraph{Support-conditioned generation.}
FMNP \citep{abuhamad2025fmnp} is closest in problem setting but models conditional function laws, not density transport. DiScoFormer \citep{ilin2026discoformer} (concurrent) estimates densities and scores directly from context using transformers; our work is complementary, characterizing the velocity-field structure that cross-attention induces rather than training a new estimator. Few-shot diffusion \citep{giannone2022fewshot}, Context Diffusion \citep{ivanova2024context}, SuTI \citep{chen2023suti}, and IP-Adapter \citep{ye2023ipadapter} condition practical generative models on reference examples via cross-attention; we connect their conditioning mechanism to NW theory and verify the connection empirically for IP-Adapter (\S\ref{sec:adapters}).

\section{Discussion and conclusion}
\label{sec:discussion}

\paragraph{Summary.}
The NW identity answers both sides of the question posed in \S\ref{sec:intro}. Cross-attention conditioning works because it implements kernel smoothing---a consistent nonparametric estimator whose quality improves with support size. On well-structured features, the plug-in field is already near-optimal. But the identity also reveals when conditioning fails: 1-NN collapse in high dimensions, geometry mismatch on anisotropic data, and support scarcity at small $m$. In each case, the learned model overcomes the failure---by maintaining multi-neighbor smoothing through adaptive projections (\S\ref{sec:mismatch}), and by amortizing over the meta-distribution to compensate for limited support (\S\ref{sec:learning}). The NW structure persists approximately in IP-Adapter (\S\ref{sec:adapters}), suggesting the framework extends to practical systems. For reference-conditioned image generation systems that condition on only a few reference images, the support scarcity regime is most relevant: the meta-learning advantage is largest at small $m$.

\paragraph{The role of $\Pi$.}
The learned model's advantage depends on the meta-distribution $\Pi$. To test this, we train on three of the four $\R^2$ families and evaluate on the held-out family (Appendix~\ref{app:support}). When the held-out family shares geometric structure with the training families (e.g., spirals held out from GMMs, rings, and moons), the learned model transfers well ($2.3\times$ better than the plug-in). When the held-out family is structurally dissimilar (e.g., cluster-based GMMs held out from curve-based families), transfer fails catastrophically ($37\times$ worse). The three curve-based families (rings, moons, spirals) generalize to each other, but cluster-based GMMs do not transfer to or from curves (Figure~\ref{fig:app-lofo}). One possible explanation is shared manifold structure: families whose supports lie on topologically similar low-dimensional manifolds may provide useful priors for each other, while families with fundamentally different support geometry (clusters vs.\ curves) do not. Formalizing this hypothesis and characterizing which properties of $\Pi$ enable amortization remain open problems.

\paragraph{Scope and extensions.}
The formal results assume Gaussian-kernel attention and the Gaussian OT path. Extending to VP paths is natural: any linear Gaussian path $X_t = \mu_t X_1 + \sigma_t X_0$ produces a Gaussian convolution of the scaled data distribution, so a NW-like velocity decomposition may exist with bandwidth $h(t) = \sigma_t/\mu_t$. Formalizing the connection to SDPA, which we demonstrate empirically in \S\ref{sec:adapters}, is another natural direction. Two open problems remain: providing a formal bound on the learned model's improvement over the plug-in, and determining the optimal noise schedule for finite support sets, balancing over-smoothing at early times against 1-NN collapse at late times.


{
\small
\bibliographystyle{plainnat}

\begin{thebibliography}{30}

\bibitem[Albergo \& Vanden-Eijnden(2023)]{albergo2023stochastic}
Albergo, M.~S. \& Vanden-Eijnden, E.
\newblock Building normalizing flows with stochastic interpolants.
\newblock In \emph{ICLR}, 2023.

\bibitem[Abu~Hamad \& Rosenbaum(2025)]{abuhamad2025fmnp}
Abu~Hamad, H. \& Rosenbaum, D.
\newblock Flow matching neural processes.
\newblock In \emph{NeurIPS}, 2025.

\bibitem[Kunkel \& Trabs(2025)]{kunkel2025minimax}
Kunkel, L. \& Trabs, M.
\newblock On the minimax optimality of flow matching through the connection to kernel density estimation.
\newblock \emph{arXiv preprint arXiv:2504.13336}, 2025.

\bibitem[Zhou et~al.(2026)]{zhou2026discretization}
Zhou, Z., Zhang, Z., \& Amini, A.~A.
\newblock Flow matching generalizes through discretization bias.
\newblock Submitted to \emph{ICLR}, 2026.

\bibitem[Bengio et~al.(2005)]{bengio2005curse}
Bengio, Y., Delalleau, O., \& {Le Roux}, N.
\newblock The curse of highly variable functions for local kernel machines.
\newblock In \emph{NeurIPS}, 2005.

\bibitem[Beyer et~al.(1999)]{beyer1999when}
Beyer, K., Goldstein, J., Ramakrishnan, R., \& Shaft, U.
\newblock When is ``nearest neighbor'' meaningful?
\newblock In \emph{ICDT}, 1999.

\bibitem[Chen et~al.(2023)]{chen2023suti}
Chen, W., Hu, H., Li, Y., Ruiz, N., Jia, X., Chang, M.-W., \& Cohen, W.~W.
\newblock Subject-driven text-to-image generation via apprenticeship learning.
\newblock In \emph{NeurIPS}, 2023.

\bibitem[Choromanski et~al.(2021)]{choromanski2021rethinking}
Choromanski, K., Likhosherstov, V., Dohan, D., Song, X., Gane, A., Sarlos, T., Hawkins, P., Davis, J., Mohiuddin, A., Kaiser, L., Belanger, D., Colwell, L., \& Weller, A.
\newblock Rethinking attention with {P}erformers.
\newblock In \emph{ICLR}, 2021.

\bibitem[Efron(2011)]{efron2011tweedie}
Efron, B.
\newblock Tweedie's formula and selection bias.
\newblock \emph{Journal of the American Statistical Association}, 106(496):1602--1614, 2011.

\bibitem[Fukumizu et~al.(2025)]{fukumizu2025flow}
Fukumizu, K., Suzuki, T., Isobe, N., Oko, K., \& Koyama, M.
\newblock Flow matching achieves almost minimax optimal convergence.
\newblock In \emph{ICLR}, 2025.

\bibitem[Goel \& Bartlett(2024)]{goel2024can}
Goel, G. \& Bartlett, P.
\newblock Can a transformer represent a {K}alman filter?
\newblock In \emph{L4DC} (PMLR vol.~242), 2024.

\bibitem[Giannone et~al.(2022)]{giannone2022fewshot}
Giannone, G., Nielsen, D., \& Winther, O.
\newblock Few-shot diffusion models.
\newblock \emph{arXiv preprint arXiv:2205.15463}, 2022.

\bibitem[Ilin \& Sushko(2025)]{ilin2026discoformer}
Ilin, V. \& Sushko, P.
\newblock {DiScoFormer}: Plug-in density and score estimation with transformers.
\newblock \emph{arXiv preprint arXiv:2511.05924}, 2025.

\bibitem[Najdenkoska et~al.(2024)]{ivanova2024context}
Najdenkoska, I., Sinha, A., Dubey, A., Mahajan, D., Ramanathan, V., \& Radenovic, F.
\newblock Context diffusion: In-context aware image generation.
\newblock In \emph{ECCV}, 2024.

\bibitem[Katharopoulos et~al.(2020)]{katharopoulos2020transformers}
Katharopoulos, A., Vyas, A., Pappas, N., \& Fleuret, F.
\newblock Transformers are {RNN}s: Fast autoregressive transformers with linear attention.
\newblock In \emph{ICML}, 2020.

\bibitem[Lepski et~al.(1997)]{lepski1997optimal}
Lepski, O., Mammen, E., \& Spokoiny, V.
\newblock Optimal spatial adaptation to inhomogeneous smoothness: An approach based on kernel estimates with variable bandwidth selectors.
\newblock \emph{Annals of Statistics}, 25(3):929--947, 1997.

\bibitem[Lipman et~al.(2023)]{lipman2023flow}
Lipman, Y., Chen, R.~T.~Q., Ben-Hamu, H., Nickel, M., \& Le, M.
\newblock Flow matching for generative modeling.
\newblock In \emph{ICLR}, 2023.

\bibitem[Liu et~al.(2025)]{liu2025oracle}
Liu, H., Liu, J., Li, Y., Bai, L., Ji, Y., Guo, Y., Wan, S., \& Wen, H.
\newblock From navigation to refinement: Revealing the two-stage nature of flow-based diffusion models through oracle velocity.
\newblock \emph{arXiv preprint arXiv:2512.02826}, 2025.

\bibitem[Kim et~al.(2025)]{lyu2025memorization}
Kim, J., Kim, S., \& Lee, J.-S.
\newblock How diffusion models memorize.
\newblock \emph{arXiv preprint arXiv:2509.25705}, 2025.

\bibitem[Meunier et~al.(2025)]{meunier2023nonlinear}
Meunier, D., Li, Z., Gretton, A., \& Kpotufe, S.
\newblock Nonlinear meta-learning can guarantee faster rates.
\newblock \emph{SIAM Journal on Mathematics of Data Science}, 7(4):1594--1615, 2025.

\bibitem[Nadaraya(1964)]{nadaraya1964estimating}
Nadaraya, E.~A.
\newblock On estimating regression.
\newblock \emph{Theory of Probability and its Applications}, 9(1):141--142, 1964.

\bibitem[Oquab et~al.(2024)]{oquab2024dinov2}
Oquab, M., Darcet, T., Moutakanni, T., Vo, H.~V., Szafraniec, M., Khalidov, V., Fernandez, P., Haziza, D., Massa, F., El-Nouby, A., Assran, M., Ballas, N., Galuba, W., Howes, R., Huang, P., Li, S., Misra, I., Rabbat, M., Sharma, V., Synnaeve, G., Xu, H., J{\'e}gou, H., Mairal, J., Labatut, P., Joulin, A., \& Bojanowski, P.
\newblock {DINOv2}: Learning robust visual features without supervision.
\newblock \emph{TMLR}, 2024.

\bibitem[Stone(1982)]{stone1982optimal}
Stone, C.~J.
\newblock Optimal global rates of convergence for nonparametric regression.
\newblock \emph{Annals of Statistics}, 10(4):1040--1053, 1982.

\bibitem[Tsybakov(2009)]{tsybakov2009introduction}
Tsybakov, A.~B.
\newblock \emph{Introduction to Nonparametric Estimation}.
\newblock Springer, 2009.

\bibitem[Tsai et~al.(2019)]{tsai2019transformer}
Tsai, Y.-H.~H., Bai, S., Yamada, M., Morency, L.-P., \& Salakhutdinov, R.
\newblock Transformer dissection: A unified understanding of transformer's attention via the lens of kernel.
\newblock In \emph{EMNLP}, 2019.

\bibitem[Watson(1964)]{watson1964smooth}
Watson, G.~S.
\newblock Smooth regression analysis.
\newblock \emph{Sankhy\={a}: The Indian Journal of Statistics, Series A}, 26(4):359--372, 1964.

\bibitem[Ye et~al.(2023)]{ye2023ipadapter}
Ye, H., Zhang, J., Liu, S., Han, X., \& Yang, W.
\newblock {IP-Adapter}: Text compatible image prompt adapter for text-to-image diffusion models.
\newblock \emph{arXiv preprint arXiv:2308.06721}, 2023.

\end{thebibliography}

}

\newpage
\appendix

\section{Proofs}
\label{app:proofs}

All results are for the OT conditional path $X_t = tX_1 + \sigma_t X_0$ with $\sigma_t = 1 - (1-\smin)t$ and $0 < \smin \le 1$. We write $\phi_h(z) = (2\pi h^2)^{-d/2}\exp(-\|z\|^2/2h^2)$ and $h(t) = \sigma_t/t$.

\subsection{Building blocks}
\label{app:building-blocks}

The following intermediate results are used in the proofs of Propositions~\ref{prop:kde} and~\ref{prop:velocity}.

\begin{proposition}[Marginal density]\label{prop:app-marginal}
The marginal density of $X_t$ is $p_t(x) = (q_t * \phi_{\sigma_t})(x)$ where $q_t(y) = t^{-d}q(y/t)$.
\end{proposition}
\begin{proof}
$X_t \mid X_1 = x_1 \sim \cN(tx_1, \sigma_t^2 I_d)$. Marginalizing over $X_1 \sim q$ and substituting $z = tx_1$, we obtain $p_t(x) = \int \phi_{\sigma_t}(x-z)\,t^{-d}q(z/t)\,dz = (q_t * \phi_{\sigma_t})(x)$.
\end{proof}

\begin{proposition}[Tweedie identity for OT path]\label{prop:app-tweedie}
$\E[tX_1 \mid X_t = x] = x + \sigma_t^2\nabla\log p_t(x)$.
\end{proposition}
\begin{proof}
Write $X_t = M + \sigma_t Z$ with $M = tX_1$ and $Z \sim \cN(0,I_d)$. Differentiating $p_t(x) = \int \phi_{\sigma_t}(x-\mu)\,q_t(\mu)\,d\mu$, we get $\nabla p_t(x) = \frac{p_t(x)}{\sigma_t^2}(\E[M \mid X_t = x] - x)$. Dividing by $p_t(x)$ and rearranging gives the result.
\end{proof}

\begin{proposition}[Velocity--score identity]\label{prop:app-velscore}
$u_t(x) = x/t + (\sigma_t/t)\nabla\log p_t(x)$.
\end{proposition}
\begin{proof}
The marginal velocity is $u_t(x) = (\E[X_1 \mid X_t = x] - (1-\smin)x)/\sigma_t$. Substituting $\E[X_1 \mid X_t = x] = x/t + (\sigma_t^2/t)\nabla\log p_t(x)$ from Proposition~\ref{prop:app-tweedie} (divided by $t$), we get
\[
u_t(x) = \frac{1}{\sigma_t}\Bigl(\frac{x}{t} + \frac{\sigma_t^2}{t}\nabla\log p_t(x) - (1-\smin)x\Bigr).
\]
The coefficient of $x$ is $(1/\sigma_t)(1/t - (1-\smin)) = (1/\sigma_t)(\sigma_t/t) = 1/t$, using $\sigma_t = 1 - (1-\smin)t$.
\end{proof}

\begin{proposition}[De-scaled marginal and score]\label{prop:app-descaled}
Define $\tilde{p}_t(\tilde{x}) = t^d p_t(t\tilde{x})$. Then: (i) $\tilde{p}_t(\tilde{x}) = (q * \phi_{h(t)})(\tilde{x})$, and (ii) $\nabla_{\tilde{x}}\log\tilde{p}_t(\tilde{x}) = (t/h(t))(u_t(t\tilde{x}) - \tilde{x})$.
\end{proposition}
\begin{proof}
(i) Using $t^d\phi_{\sigma_t}(tz) = \phi_{\sigma_t/t}(z)$, we get $t^d p_t(t\tilde{x}) = \int \phi_{h(t)}(\tilde{x} - x_1)\,q(x_1)\,dx_1$. (ii) Chain rule gives $\nabla_{\tilde{x}}\log\tilde{p}_t = t\nabla_x\log p_t(t\tilde{x})$. By Proposition~\ref{prop:app-velscore}, $\nabla\log p_t(x) = (tu_t(x) - x)/\sigma_t$, so at $x = t\tilde{x}$: $t \cdot t(u_t(t\tilde{x}) - \tilde{x})/\sigma_t = (t/h(t))(u_t(t\tilde{x}) - \tilde{x})$.
\end{proof}

\subsection{Proof of Proposition~\ref{prop:kde} (de-scaled empirical OT path is KDE)}
\label{app:proof-kde}

\begin{proof}
Direct from Proposition~\ref{prop:app-descaled}(i) with $q = \hat{q}_S = \frac{1}{m}\sum_i \delta_{s_i}$:
$\tilde{p}_t^S(\tilde{x}) = \frac{1}{m}\sum_{i=1}^m \phi_{h(t)}(\tilde{x} - s_i)$.
\end{proof}

\subsection{Proof of Proposition~\ref{prop:velocity} (exact empirical velocity field)}
\label{app:proof-velocity}

The proof uses the score of the de-scaled empirical path:
$\nabla\log\tilde{p}_t^S(\tilde{x}) = (m_h(\tilde{x};S) - \tilde{x})/h^2$,
which follows from differentiating the KDE and recognizing the NW weights.

\begin{proof}
By Proposition~\ref{prop:app-velscore} applied to $p_t^S$ at $x = t\tilde{x}$: $u_t^S(t\tilde{x}) = \tilde{x} + (\sigma_t/t^2)\nabla\log\tilde{p}_t^S(\tilde{x})$. Substituting the score: $= \tilde{x} + (\sigma_t/t^2)(m_h(\tilde{x};S) - \tilde{x})/h^2$. Since $h = \sigma_t/t$, we have $t^2 h^2 = \sigma_t^2$, so $\sigma_t/(t^2 h^2) = 1/\sigma_t$.
\end{proof}

\subsection{Velocity MSE bound (\S\ref{sec:collapse})}
\label{app:vel-mse}

Let $S = \{s_i\}_{i=1}^m \overset{\text{iid}}{\sim} q$ and fix $t \in (0,1]$. Write $h = h(t)$ and let $\tilde{x}$ be an interior point of $\mathrm{supp}(q)$. We require:
\begin{enumerate}[nosep,label=(\roman*)]
  \item $q$ has a Lebesgue density $f$ with $f(\tilde{x}) > 0$.
  \item $f$ is continuous in a neighborhood of $\tilde{x}$.
  \item The regime $m\,h^d \to \infty$ (enough support points per kernel volume).
\end{enumerate}
Under (i)--(iii), the NW local mean $m_h^S(\tilde{x})$ is a ratio of two sample means whose denominators converge to $f_h(\tilde{x}) = (f * \phi_h)(\tilde{x}) > 0$. The variance of the numerator is $O(1/(m\,h^d))$ by the law of large numbers for kernel averages \citep{tsybakov2009introduction}, and the delta method for ratios gives $\mathrm{Var}(m_h^S(\tilde{x})) = O(1/(m\,h^d\,f(\tilde{x})))$. Since $u_t^S(t\tilde{x}) - u_t^q(t\tilde{x}) = (m_h^S(\tilde{x}) - m_h^q(\tilde{x}))/\sigma_t$ by Proposition~\ref{prop:velocity}, the velocity MSE is $\mathrm{Var}(m_h^S)/\sigma_t^2 = O(t^d/(m\,\sigma_t^{d+2}\,f(\tilde{x})))$. No smoothness beyond continuity is needed because both the sample and population NW means use the same bandwidth; there is no bias term.

\subsection{Proof of Theorem~\ref{thm:realization} (attention realization)}
\label{app:proof-realization}

\begin{proof}
The softmax weights with logits $\ell_i = -\|\tilde{x} - s_i\|^2/(2h^2)$ are $\alpha_i = \phi_h(\tilde{x} - s_i)/\sum_j \phi_h(\tilde{x} - s_j) = w_i(\tilde{x})$, so $\text{Attn} = m_h(\tilde{x};S)$. Applying $A_t$: $\tilde{x} + (m_h - \tilde{x})/\sigma_t = u_t^S(t\tilde{x})$ by Proposition~\ref{prop:velocity}.
\end{proof}

\begin{corollary}[Dot-product realization]\label{cor:dotproduct}
The Gaussian-kernel logits equal unscaled dot products via a nonlinear feature lift into $\R^{d+2}$. Define $Q(\tilde{x},t) = [\tilde{x}/h,\;-\|\tilde{x}\|^2/(2h^2),\;1]$ and $K(s_i,t) = [s_i/h,\;1,\;-\|s_i\|^2/(2h^2)]$. Then $Q \cdot K = -\|\tilde{x} - s_i\|^2/(2h^2) = \ell_i$. Note: these feature maps are nonlinear and carry no $1/\sqrt{\dk}$ scaling, so this is distinct from standard SDPA.
\end{corollary}

\subsection{Bayes optimality and plug-in distributional rate (\S\ref{sec:learning})}
\label{app:bayes}

The ICFM loss~\eqref{eq:icfm-loss} is a squared-error regression with conditioning on $S$. The unconstrained minimizer is the conditional expectation $v^*(x,t,S) = \E[Y_t \mid X_t = x, t, S]$---the posterior mean velocity given the support set and current state. The plug-in field (Proposition~\ref{prop:velocity}) is Bayes-optimal under the empirical measure $\hat{q}_S$ for a single task, but generally not under the meta-distribution $\Pi$: the meta-loss Bayes rule uses the posterior predictive $\bar{q}_S = \E[q \mid S]$ in place of $\hat{q}_S$, and the plug-in matches this only when the posterior concentrates on the empirical measure.

\begin{corollary}[Plug-in distributional rate]\label{cor:plugin-rate}
The ODE flow of $u_t^S$ from $\cN(0,I_d)$ produces samples distributed as $\hat{q}_S * \phi_{\smin}$: the plug-in output is a Gaussian KDE at bandwidth $\smin$.
\end{corollary}
\begin{proof}
By Proposition~\ref{prop:kde} at $t = 1$: $p_1^S(x) = \frac{1}{m}\sum_i \phi_{\smin}(x - s_i)$. The velocity $u_t^S$ generates $p_t^S$ via the continuity equation, so the ODE flow transports $p_0 = \cN(0,I_d)$ to $p_1^S = \hat{q}_S * \phi_{\smin}$.
\end{proof}

\section{Multi-head theory}
\label{app:multihead}

\begin{definition}[Generalized NW estimator]\label{def:gen-nw}
For any positive kernel $K: \mathcal{X} \times \mathcal{X} \to (0,\infty)$, the generalized NW estimator is $m_K(x; \{(x_i,y_i)\}) = \sum_i w_i^K(x)\,y_i$ with $w_i^K(x) = K(x,x_i)/\sum_j K(x,x_j)$.
\end{definition}

\begin{proof}[Proof of Proposition~\ref{prop:multihead}]
By the MHA definition, $\text{MHA}(q,Z) = \sum_{h=1}^H W_O^h\,\text{Head}_h(q,Z)$. Each head's attention weights $\alpha_i^{(h)} = \exp((W_Q^h q)^\top W_K^h z_i/\sqrt{\dk})/\sum_j(\cdots) = w_i^{K_h}(q)$ with $K_h(x,s) = \exp(x^\top W_Q^{h\top}W_K^h s/\sqrt{\dk})$. The output $\text{Head}_h = \sum_i \alpha_i^{(h)}W_V^h z_i = m_{K_h}(q; \{(z_i, W_V^h z_i)\})$.
The logit matrix $A_h = W_Q^{h\top}W_K^h/\sqrt{\dk} \in \R^{\dmodel \times \dmodel}$ has $\text{rank}(A_h) \le \dk$.
\end{proof}

\begin{proposition}[Spherical NW rate]\label{prop:sphere-rate}
Under QK-normalization ($\|\tilde{q}\| = \|\tilde{k}_i\| = R$ for all $i$, $\dk \ge 2$), the attention kernel is vMF-zonal on $S^{\dk-1}(R)$. With tunable concentration $\kappa_m \asymp m^{2/(\dk+3)}$, the NW estimator achieves $\text{MSE}^* = O(m^{-4/(\dk+3)})$, the Stone (1982) minimax rate in $\dk - 1$ intrinsic dimensions.
\end{proposition}
\begin{proof}
Under spherical normalization, the logits become $\tilde{q}^\top\tilde{k}_i/\sqrt{\dk} = \kappa\cos\theta_i$ where $\theta_i$ is the geodesic angle between $\tilde{q}$ and $\tilde{k}_i$ on $S^{\dk-1}(R)$, and $\kappa = R^2/\sqrt{\dk}$. The kernel $K(\theta) = \exp(\kappa\cos\theta)$ is the von Mises--Fisher (vMF) density up to normalization, which is zonal on the sphere.

For the NW bias: the kernel average $m_K(q) = \sum_i w_i y_i$ with $w_i \propto \exp(\kappa\cos\theta_i)$ concentrates on a spherical cap of angular radius $O(1/\sqrt{\kappa})$. A Laplace approximation of the vMF integral gives an effective bandwidth $h_{\text{eff}} \sim 1/\sqrt{\kappa}$ on the $(\dk-1)$-dimensional sphere, yielding bias $O(h_{\text{eff}}^2) = O(1/\kappa)$ for Lipschitz targets.

For the NW variance: the effective number of points in the cap scales as $m \cdot \text{cap volume} \sim m \cdot h_{\text{eff}}^{\dk-1} = m/\kappa^{(\dk-1)/2}$, giving variance $O(\kappa^{(\dk-1)/2}/m)$.

Balancing bias$^2$ and variance: $1/\kappa^2 = \kappa^{(\dk-1)/2}/m$ gives $\kappa^* \asymp m^{2/(\dk+3)}$, and MSE$^* = O(m^{-4/(\dk+3)})$.
\end{proof}

\section{Anisotropic extension}
\label{app:aniso}

For $M \succ 0$, define the anisotropic NW weights $w_i^{(h,M)}(x;S) = \exp(-\|x - s_i\|_M^2/2h^2)/\sum_j\exp(-\|x - s_j\|_M^2/2h^2)$ with $\|z\|_M^2 = z^\top M z$, and the anisotropic local mean $m_{h,M}(x;S) = \sum_i w_i^{(h,M)}(x;S)\,s_i$.

\begin{theorem}[Anisotropic attention realization]\label{thm:aniso}
Let $M \succ 0$. Under the OT path with base $X_0 \sim \cN(0, M^{-1})$, the anisotropic empirical velocity field is:
\begin{equation}
  u_t^{S,M}(t\tilde{x}) = \tilde{x} + \frac{m_{h(t),M}(\tilde{x};S) - \tilde{x}}{\sigma_t}.
\end{equation}
This has exactly the same affine form as the isotropic case, with $m_{h,M}$ replacing $m_h$. A Mahalanobis-kernel cross-attention head with logits $\ell_i = -\|\tilde{x} - s_i\|_M^2/(2h^2)$, followed by the same post-map $A_t$, exactly realizes this velocity.
\end{theorem}

\begin{proof}
Under the anisotropic base $X_0 \sim \cN(0, M^{-1})$, the marginal at time $t$ is $p_t^{S,M}(x) = \frac{1}{m}\sum_i \cN(x \mid ts_i, \sigma_t^2 M^{-1})$. The Tweedie identity for this model gives $\E[tX_1 \mid X_t = x] = x + \sigma_t^2 M^{-1} \nabla\log p_t^{S,M}(x)$, so the velocity--score identity becomes:
\[
u_t^{S,M}(x) = \frac{x}{t} + \frac{\sigma_t}{t} M^{-1}\nabla\log p_t^{S,M}(x).
\]
De-scaling ($\tilde{x} = x/t$) and differentiating the anisotropic KDE $\tilde{p}_t^{S,M}(\tilde{x}) = \frac{1}{m}\sum_i \phi_{h,M}(\tilde{x} - s_i)$ give:
\[
\nabla\log\tilde{p}_t^{S,M}(\tilde{x}) = \frac{M}{h^2}(m_{h,M}(\tilde{x};S) - \tilde{x}).
\]
Substituting into the velocity--score identity, the $M^{-1}$ cancels the $M$:
\[
u_t^{S,M}(t\tilde{x}) = \tilde{x} + \frac{\sigma_t}{t^2 h^2}(m_{h,M}(\tilde{x};S) - \tilde{x}) = \tilde{x} + \frac{m_{h,M}(\tilde{x};S) - \tilde{x}}{\sigma_t},
\]
which is the same affine post-map $A_t$ as in the isotropic case.
\end{proof}

\paragraph{Low-rank extension.} When $M = L^\top L$ with $L \in \R^{r \times d}$ and $r < d$, the theorem as stated requires $M \succ 0$ and does not directly apply. However, the Mahalanobis distance $\|\tilde{x} - s_i\|_M^2 = \|L(\tilde{x} - s_i)\|^2$ depends only on the $r$-dimensional projections $L\tilde{x}$ and $Ls_i$; adding a small ridge $\epsilon I$ to $M$ recovers positive definiteness while preserving the low-rank kernel structure in the limit $\epsilon \to 0$. The effective estimation dimension is then $r$, not $d$.

\section{Extended experiments}
\label{app:extended}

\subsection{Training details}
\label{app:training}

All models use $\dmodel = 128$, $H = 4$ learned cross-attention heads, 3 transformer layers, and $\smin = 0.01$. Training uses AdamW with learning rate $10^{-4}$ and weight decay $10^{-5}$ for 100k steps. Each training step samples a task $q \sim \Pi$, draws a support set $S \sim q^m$ and a target point $X_1 \sim q$, samples noise $X_0 \sim \cN(0, I_d)$ and time $t \sim \mathrm{Uniform}(0,1]$, and regresses the velocity target $Y_t = X_1 - (1-\smin)X_0$. Enabling the frozen exact head adds a learned projection and widens the output MLP ($+2.6\%$ parameters). Checkpoints are selected by best validation MMD$^2$. All training was conducted on $11\times$ RTX 4090 GPUs; total compute was $\sim\!350$ GPU-hours across all experiments.

\subsection{Collapse and mismatch: additional data (\S\ref{sec:collapse}, \S\ref{sec:mismatch})}
\label{app:collapse}

Figure~\ref{fig:app-collapse-mismatch}(a) shows the ImageNet transition curve: the exact head is marginal at $d \le 16$ and harmful at $d = 64$ ($3/5$ ON seeds diverge vs.\ $0/3$ OFF). Figure~\ref{fig:app-collapse-mismatch}(b) shows the shells comparison at $d = 8, 16, 32$: the learned model outperforms the plug-in by $6$--$11\times$ despite healthy $\neff$.

\begin{figure}[ht!]
  \centering
  \includegraphics[width=\linewidth]{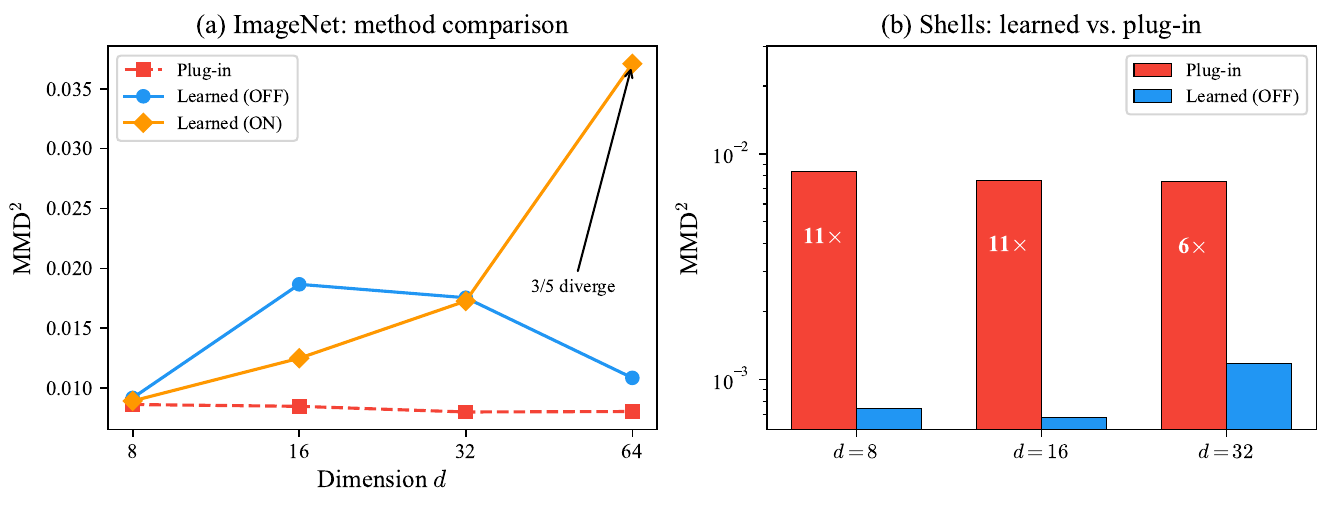}
  \vspace{-14pt}
  \caption{\textbf{(a)}~ImageNet DINOv2+PCA: method comparison across $d$. The exact head (ON) is marginal at $d \le 16$ and destabilizes training at $d=64$. \textbf{(b)}~Spherical shells: the plug-in fails $6$--$11\times$ despite healthy $\neff$, confirming geometry mismatch.}
  \label{fig:app-collapse-mismatch}
\end{figure}

Figure~\ref{fig:app-whitened} shows the whitened ImageNet control: five whitening strengths interpolating from original ($\neff = 12.35$) to fully whitened ($\neff = 1.05$). Plug-in MMD$^2$ is flat ($0.0075$--$0.0077$), confirming that $\neff$ collapse alone does not degrade the plug-in.

\begin{figure}[ht!]
  \centering
  \includegraphics[width=0.8\linewidth]{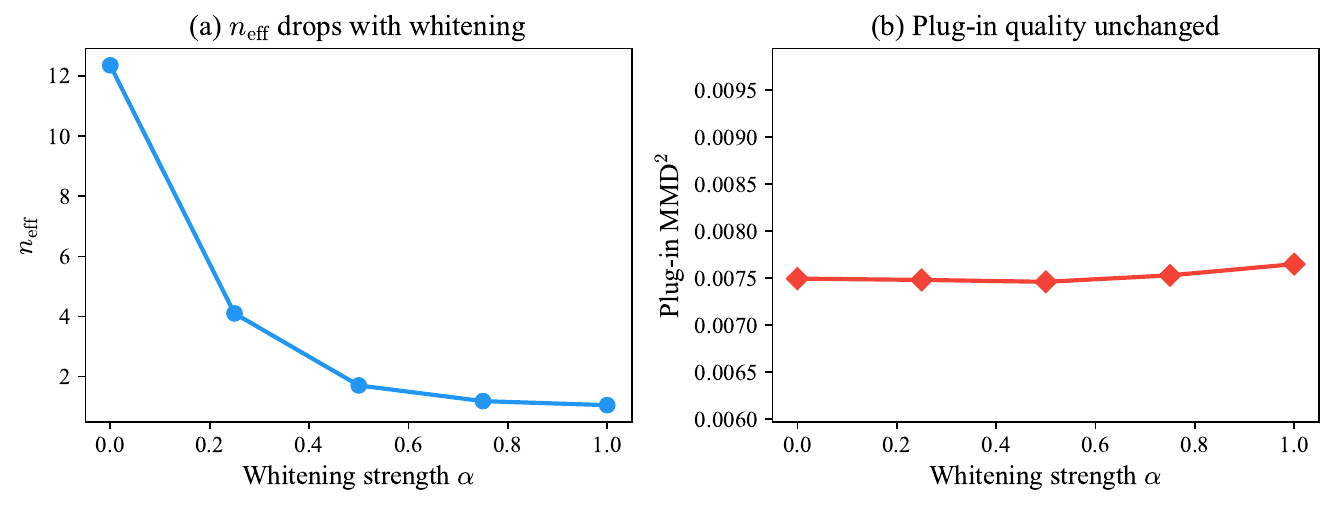}
  \vspace{-4pt}
  \caption{Whitened ImageNet sufficiency test. \textbf{(a)}~$\neff$ drops $12\times$ with whitening. \textbf{(b)}~Plug-in MMD$^2$ is unchanged, confirming $\neff$ alone is not causal.}
  \label{fig:app-whitened}
\end{figure}

\subsection{Support scaling and LOFO (\S\ref{sec:learning})}
\label{app:support}

Figure~\ref{fig:app-support-sweep}(a) shows the $\R^2$ GMM support sweep: MMD$^2 \sim m^{-0.83}$ ($R^2 = 0.97$), with the learned model maintaining $2.4$--$5.5\times$ advantage at all $m$. Figure~\ref{fig:app-support-sweep}(b) shows the shells sweep at $d=8$: $\alpha_{\text{shells}} = 0.669 \approx \alpha_{\text{GMM}} = 0.631$, confirming that the scaling rate is family-independent.

\begin{figure}[b]
  \centering
  \includegraphics[width=\linewidth]{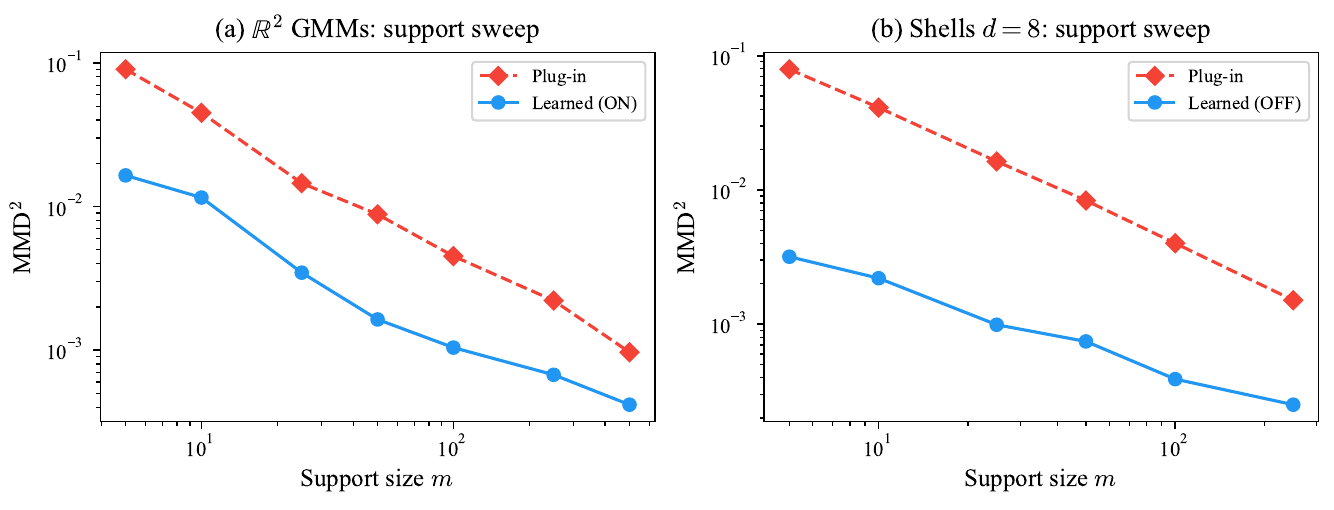}
  \vspace{-14pt}
  \caption{\textbf{(a)}~$\R^2$ GMMs: learned maintains $2.4$--$5.5\times$ advantage at all $m$. \textbf{(b)}~Shells $d=8$: $\alpha_{\text{shells}} \approx \alpha_{\text{GMM}}$; the rate bottleneck is meta-learning, not task complexity.}
  \label{fig:app-support-sweep}
\end{figure}

\paragraph{LOFO generalization.} Leave-one-family-out on $\R^2$: spirals transfer well ($2.3\times$ better than plug-in from $\{\text{GMM, rings, moons}\}$); GMMs fail catastrophically ($37\times$ worse from $\{\text{rings, moons, spirals}\}$). Generalization is structure-dependent. Figure~\ref{fig:app-lofo} shows the full comparison.

\begin{figure}[b]
  \centering
  \includegraphics[width=0.55\linewidth]{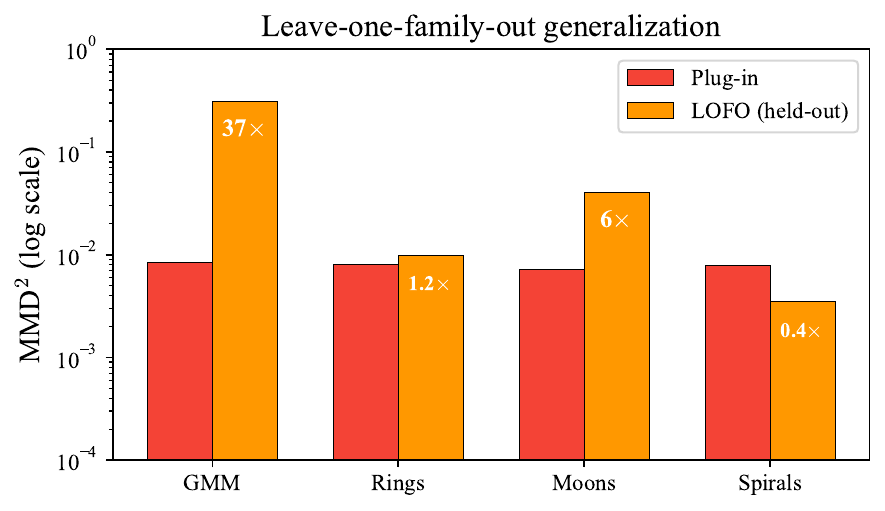}
  \vspace{-4pt}
  \caption{Leave-one-family-out: plug-in MMD$^2$ (red) vs.\ LOFO held-out MMD$^2$ (orange). GMMs held out from curve families: $37\times$ worse. Spirals held out: $0.4\times$ (better than plug-in).}
  \label{fig:app-lofo}
\end{figure}

\subsection{Multi-head analysis}
\label{app:heads}

Figure~\ref{fig:app-attention} shows per-head attention entropy vs.\ flow time at $d=8$, $H=8$. Heads develop distinct entropy profiles (spread 4.58 nats), with some attending sharply (entropy $\approx 0.3$) and others diffusely (entropy $\approx 5.2$). $W_Q$ subspace overlap decreases with $H$, confirming that heads learn nearly orthogonal projections.

\paragraph{QK-norm ablation.} QK-normalization improves $\alpha$ by $+0.05$--$0.08$ at every $H$ and restores monotonic $\alpha$ ordering. However, it compresses $\alpha$ separation to $0.011$ (vs.\ $0.038$ standard).

\begin{figure}[b]
  \centering
  \includegraphics[width=0.7\linewidth]{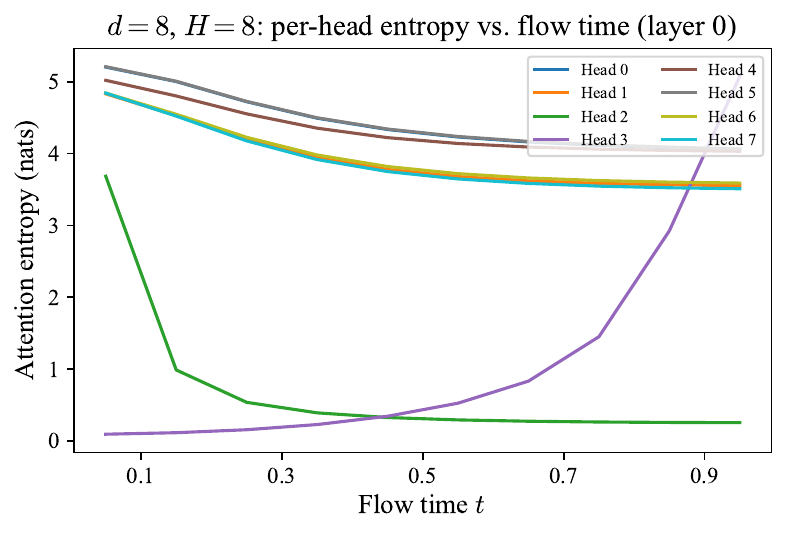}
  \vspace{-4pt}
  \caption{Per-head attention entropy at $d=8$, $H=8$ (layer 0). Heads specialize to different effective bandwidths: some maintain high entropy (broad smoothing) while others collapse to near-zero (sharp, nearest-neighbor-like).}
  \label{fig:app-attention}
\end{figure}

\subsection{Controls and NW variance scaling}
\label{app:controls}
\label{app:nw-rate}

\paragraph{ODE solver.} dopri5 vs.\ Euler-100: plug-in MMD$^2$ improves only $2.4\%$ on $\R^2$ GMMs.

\paragraph{Capacity.} $\dmodel = 256$ gives $<5\%$ change at $d=4$ and $d=64$. Capacity is not the bottleneck.

\paragraph{Time sampling.} Log-uniform bandwidth sampling shows no benefit ($p = 0.67$, 4 seeds).

\paragraph{NW variance scaling.} We verify that the NW variance follows classical theory by measuring $\|m_h^S(\tilde{x}) - m_h^{S_{\text{ref}}}(\tilde{x})\|^2$ directly at the Silverman bandwidth $h^* = m^{-1/(4+d)}$ across $m \in \{10, 25, 50, 100, 250, 500, 1000\}$, with $m_{\text{ref}} = 50{,}000$ as a population proxy.

\vspace{4pt}
\begin{center}
\small
\begin{tabular}{llcccc}
\toprule
Distribution & $d$ & $\alpha_{\text{obs}}$ & $4/(4+d)$ & Ratio & $R^2$ \\
\midrule
GMM-5 & 2 & 1.146 & 0.667 & 1.72$\times$ & 0.982 \\
GMM-5 & 4 & 0.889 & 0.500 & 1.78$\times$ & 0.973 \\
GMM-5 & 8 & 0.534 & 0.333 & 1.60$\times$ & 0.966 \\
\midrule
Fourier & 2 & 1.240 & 0.667 & 1.86$\times$ & 0.999 \\
Fourier & 4 & 0.709 & 0.500 & 1.42$\times$ & 0.999 \\
Fourier & 8 & \textbf{0.320} & \textbf{0.333} & \textbf{0.96$\times$} & 0.998 \\
\bottomrule
\end{tabular}
\end{center}
\vspace{-4pt}
The Fourier density at $d = 8$ matches the theoretical $O(m^{-4/(4+d)})$ rate to within $4\%$; GMMs exceed the theoretical rate because they are parametric.

\subsection{IP-Adapter extended analysis and sample visualizations}
\label{app:samples}

Figure~\ref{fig:app-ipadapter-null} extends the IP-Adapter analysis from \S\ref{sec:adapters} with null-model comparisons. The observed $\rho = 0.885$ far exceeds the permutation null ($0.0$) and approaches the same-NN ceiling ($1.0$). The distribution is concentrated: $56\%$ of heads exceed $\rho = 0.9$ at the final timestep, and the correlation increases monotonically with denoising.

\begin{figure}[ht!]
  \centering
  \includegraphics[width=\linewidth]{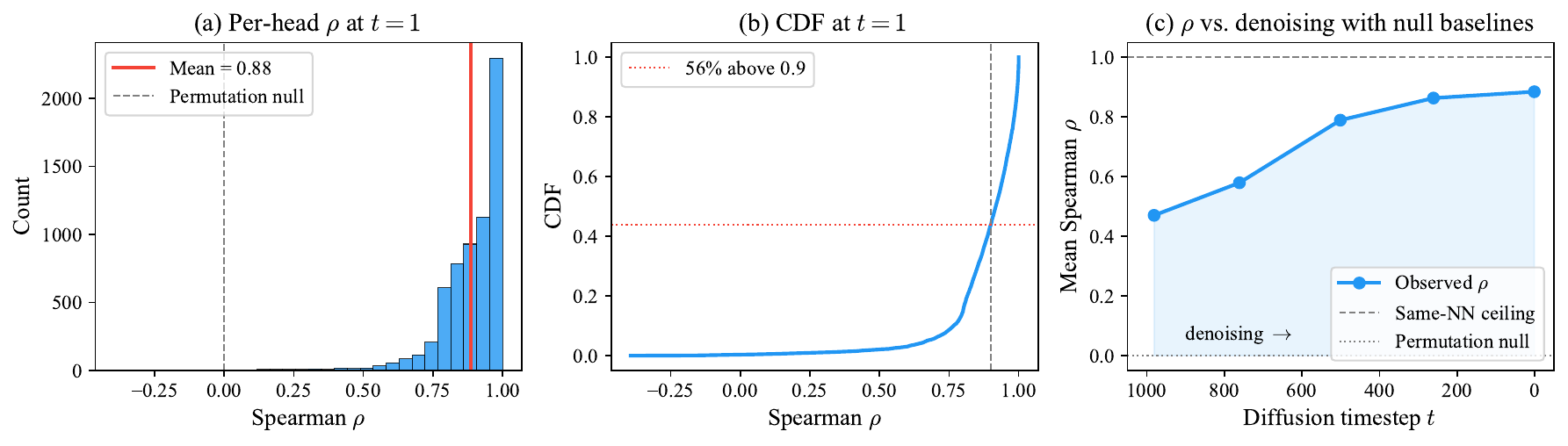}
  \vspace{-14pt}
  \caption{IP-Adapter null-model analysis (SD 1.5, 50 images, 128 heads). \textbf{(a)}~Per-head $\rho$ at $t=1$: mean $0.885$ far exceeds permutation null. \textbf{(b)}~CDF: $56\%$ of heads above $\rho = 0.9$. \textbf{(c)}~$\rho$ vs.\ denoising timestep between permutation null and same-NN ceiling.}
  \label{fig:app-ipadapter-null}
\end{figure}

\begin{figure}[ht!]
  \centering
  \includegraphics[width=\linewidth]{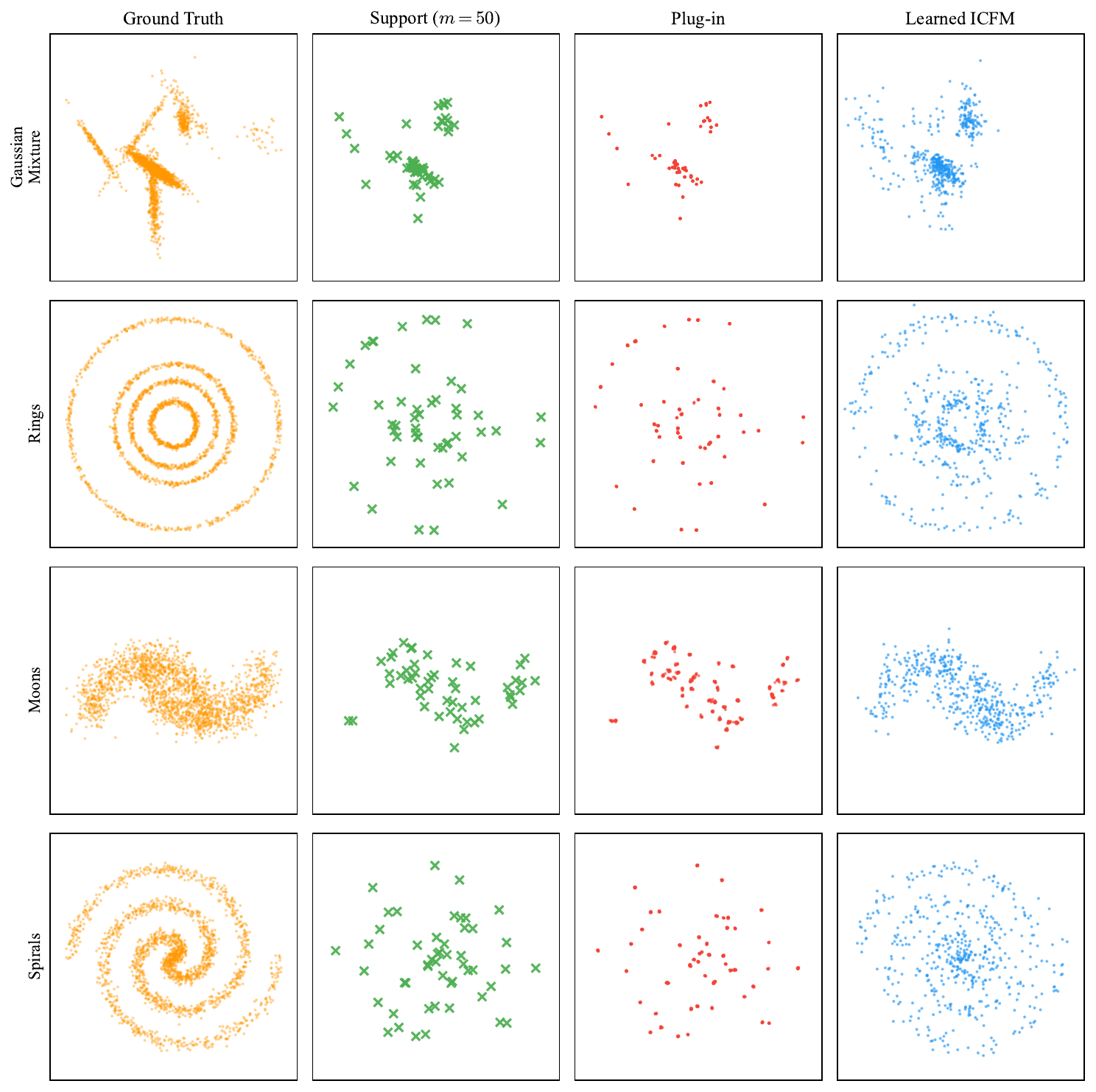}
  \vspace{-4pt}
  \caption{Support-conditioned generation on four $\R^2$ families ($m=50$). Columns: ground truth, support set, plug-in, learned ICFM. The plug-in places mass near support points; the learned model generates fresh samples matching the ground-truth density.}
  \label{fig:app-samples}
\end{figure}


\end{document}